%% file: eccv2012submission.tex
\documentclass[runningheads,a4paper]{llncs}

\usepackage{amssymb}
\usepackage{graphicx}

\usepackage{amsmath}
\usepackage{color}
\usepackage{multirow}
\usepackage{empheq}
\usepackage{url}

\renewcommand{\div}{\operatorname{div}}

\urldef{\mailsa}\path|{firstname.lastname}@liu.se|

\begin{document}

\mainmatter  % start of an individual contribution

\title{On Tensor-Based PDEs and their Corresponding Variational Formulations with Application to Color Image Denoising}

\titlerunning{On Tensor-Based PDEs and their Corresponding Variational Formulations}

\author{Freddie \AA str{\"o}m\inst{1,2}\and George Baravdish\inst{3}\and Michael Felsberg\inst{1,2}}

\authorrunning{On Tensor-Based PDEs and their Corresponding Variational Formulations}

\institute{Computer Vision Laboratory, Department of E.E., Link{\"o}ping University
\and Center for Medical Image Science and Visualization (CMIV), Link{\"o}ping University
\and Department of Science and Technology, Link{\"o}ping University \\ \mailsa} 

% \institute{Springer-Verlag, Computer Science Editorial,\\
% Tiergartenstr. 17, 69121 Heidelberg, Germany\\
% \mailsa\\
% \url{http://www.springer.com/lncs}}


\toctitle{Lecture Notes in Computer Science}
\tocauthor{Authors' Instructions}
\maketitle

\begin{abstract}
  The case when a partial differential equation (PDE) can be considered as an Euler-Lagrange (E-L) equation of an energy functional, consisting of a data term and a smoothness term is investigated. We show the necessary conditions for a PDE to be the E-L equation for a corresponding functional. This energy functional is applied to a color image denoising problem and it is shown that the method compares favorably to current state-of-the-art color image denoising techniques. 
\end{abstract}

\section{Introduction}
In their seminal work Perona and Malik proposed a non-linear diffusion PDE to filter an image while retaining lines and structure \cite{Perona90scale-spaceand}. This triggered researchers to investigate well-posedness of the underlying PDE, but also to build on the drawbacks of the Perona and Malik formulation, namely that noise is preserved along lines and edges in the image structure. A tensor-based image diffusion formulation, which gained wide acceptance, was proposed by Weickert \cite{Weickert96anisotropicdiffusion} where the outer product of the gradient was used to describe the local orientation of the image structure. As the main novelty of this work we investigate the case when a given image diffusion PDE can be considered as a corresponding E-L equation to a functional of the form
\begin{equation}
  J(u) = \frac{1}{p}||u - u_0||^p + \mathcal{R}(u) \enspace ,
  \label{eq:main_functional}
\end{equation}
where $\mathcal{R}(u)$ is a \emph{tensor-based} smoothness term, and the $L^p$-norm, $1 < p < \infty$, is a data term. Furthermore, the derived E-L equation of \eqref{eq:main_functional} is applied to a color image denoising problem.

%\subsubsection{Related work}
\noindent \textbf{Related work} State-of-the-art gray-value image denoising achieves impressive results. However, the extention from gray-valued image processing to color images still pose major problems due to the unknown mechanisms which govern the process describing color perception. One of the most common color spaces used to represent images is the RGB color space, despite it being known that the colors (red, green and blue) are correlated due to physical properties such as the design of color sensitive filters in camera equipment. In \cite{Weickert1999201} anisotropic diffusion in the RGB color space is proposed, where the color components are used to derive a weighted diffusion tensor. One drawback of this formulation is that color artifacts can appear on the boundary of regions with sharp color changes. Some alternative color spaces investigated for image diffusion include manifolds \cite{661181}, the CMY, HSV, Lab and Luv color space \cite{Renner2003}, luminance and chormaticity \cite{918563}, and a maximal decorrelation approach \cite{diva2:424137}. Due to the simplicity of its application we follow the color model used in \cite{diva2:424137} which was developed in \cite{Lenz2010}. 

% Manifold Sapiro ,541429
% Researchers have been working on relating diffusion processes to variational formulations. 

With regards to variational approaches of image diffusion processes, it has been shown that non-linear diffusion and wavelet shrinkage are both related to variational approaches \cite{springerlink:10.1007/BF00131148}. Also, anisotropic diffusion has been described using robust statistics of natural images \cite{1238435}. However, in \cite{1238435} they construct an update scheme for the iterative anisotropic process which contains the addition of a an extra convolution of the diffusion tensor. In the paper \cite{5539959} a tensor-based functional is given, but the tensor it describes is static, thus the update-scheme becomes linear. Generalizations of the variational formulation exists, \emph{i.e.} in \cite{Baravdish11} the smoothness term is considered as an $L^{p}$-norm where $p$ is an one-dimensional function. 

The main contributions of this paper are: (1) we establish a new tensor-based variational formulation for image diffusion in Theorem 1; (2) in Theorem 2 we derive necessary conditions such that it is possible to find an energy functional given a tensor-based PDE; (3) the derived E-L equation of the established image diffusion functional is applied to a color image denoising problem. It is shown that the resulting denoising compares favorably to state-of-the-art techniques such as anisotropic diffusion \cite{Weickert1999201}, tracebased diffusion \cite{5696757,diva2:424137} and BM3D \cite{4378954}. 

In section 2 we briefly review both scalar and tensor-based image diffusion filtering as a motivation for the established functional. Furthermore, we state the necessary conditions for a PDE to be considered as an E-L equation to the same functional. The derived functional is applied to a color image denoising problem in section 3. The paper is concluded with some final remarks.  

\section{Image diffusion and variational formulation}
Convolving an image with a Gaussian filter reduces the noise variance, but does not take the underlying image structure such as edges and lines into account. To avoid blurring of lines and edges Perona and Malik \cite{Perona90scale-spaceand} introduced an ``edge-stopping'' function in the solution of the heat equation
\begin{equation}
  \partial_t u = \div (g(|\nabla u|^2)\nabla u) \enspace ,
  \label{eq:divPM}
\end{equation}
with diffusivity function $g$, and $\nabla u = \begin{pmatrix} u_x, & u_y \end{pmatrix}^t$ where $u_x$ and $u_y$ denote the derivative of $u$ with respect to $x$ and $y$. If $g = I$ the heat equation is obtained. The Perona-Malik formulation produces visually pleasing denoising results and preserves the image structure for a large number of iterations. However the formulation only depends on the absolute image gradient and hence noise is preserved at edges and lines. Weickert \cite{Weickert96anisotropicdiffusion} replaced the diffusivity function with a tensor. This allows the filter to reduce noise along edges and lines and simultaneously preserve image structure. The tensor-based image diffusion formulation is 
\begin{equation}
  \partial_t u = \div (D \nabla u) \enspace ,
  \label{eq:divD}
\end{equation}
where $D = D(T_s)$ is a positive-semidefinite diffusion tensor constructed by scaling the eigenvalues of the structure tensor $T_s$,
\begin{equation}
  T_s = w *( \nabla u \nabla u^t ) \enspace ,
  \label{eq:Tw}
\end{equation}
where $*$ denotes the convolution operation and $w$ is typically a Gaussian weight function \cite{Forstner87}. % % 

\subsection{Scalar diffusion}
The functional 
\begin{equation}
  J(u) = \iint_\Omega \frac{1}{2}(u - u_0)^2 dxdy+ \lambda \iint_\Omega \Phi(|\nabla u|) \; dxdy \enspace , 
  \label{eq:functionalEX}
\end{equation}
where $\Phi$ is a strict convex function and $\lambda$ is a scalar, describes a nonlinear minimization problem for image denoising. By using the G{\^a}teaux derivative and the divergence theorem, the solution $u$ which minimizes $J(u)$ is the solution to the E-L equation
\begin{equation}
   \frac{\partial J(u)}{\partial u} = 0 \;\;\; \mbox{in} \quad \Omega, \quad \nabla u \cdot n = 0 \quad \mbox{on} \quad \partial \Omega \enspace ,
 \end{equation}
where $\Omega$ is a grid described by the image size in pixels and $n$ is the normal vector on the boundary $\partial \Omega$. The functional $J(u)$ satisfies the initial-boundary value problem for the PDE
\begin{align}
  \left\{
    \begin{array}{l l}
      \partial_t u - \div \left ( \dfrac{\Phi'(|\nabla u|)}{|\nabla u|}
        \nabla u\right ) = 0 & \quad \text{in } \Omega, t > 0  \\[1mm]
      \partial_n u = 0                                   & \quad  \text{on } \partial \Omega, t > 0  \\[1mm]
      u(x,y,0) = u_0(x,y)                                & \quad \text{in } \Omega
    \end{array} \right.
  \label{eq:div_scalar}
\end{align}
It is commonly accepted that there is no known way of determining a functional such that $\Phi$ is any arbitrary function. However, it is possible to derive a corresponding functional by equating the non-specified diffusivity function $g$ with the known E-L equation \eqref{eq:div_scalar}
\begin{equation}
  g(|\nabla u|) = \frac{\Phi'(|\nabla u|)}{|\nabla u|} \enspace ,
  \label{eq:scalardiffusion}
\end{equation}
hence, with $s = |\nabla u|$ one gets $\Phi(s) = \int sg(s)\,ds = sG(s) - G_1(s)+C$
%\begin{equation}
%  \Phi(s) = \int sg(s)\,ds = sG(s) - G_1(s)+C \enspace ,
%\end{equation}
where $G_1'=G$, $G'=g$ and the constant $C$ is chosen so that  $\Phi$ is non-negative. For example if $g(s) = s^{p-2}$ then $\Phi(s) = \frac{|\nabla u|^p}{p}$ and similarly if $g(s) = e^{-s}$ then $\Phi = -s e^{-s} - e^{-s} + C$ where $C \geq 1$. Naturally, these results carry over to other types of diffusivity functions such as the popular negative exponential diffusivity function and the Perona-Malik function $g(s)=\frac{1}{1+s}$.

% however for specific functions of interest we show that a functional does exist.

\subsection{Deriving a new tensor-based functional}
To achieve better denoising results close to lines and edges than the scalar diffusion formulation enables, it is common to define the tensor $T$ which describes local orientation of the image structure. Hence the PDE in \eqref{eq:div_scalar} reads
\begin{align}
  \left\{
    \begin{array}{l l}
      \partial_t u -  \div(T \nabla u) = 0                & \quad \text{in } \Omega, t > 0 \\[1mm]
      \partial_n u = 0                                    & \quad \text{on } \partial \Omega, t > 0  \\[1mm]
      u(x,y,0) = u_0(x,y)                                 & \quad \text{in } \Omega
    \end{array} \right.
  \label{eq:divtensor}
\end{align}
where $T(\nabla u)$ depends on the image gradient. Note that the common formulation of the tensor $T$ is to let it depend on $(x,y) \in \Omega$. The ansatz to derive the new tensor-based image diffusion functional is based on equating \eqref{eq:div_scalar} and \eqref{eq:divtensor} such that 
\begin{equation}
  \div(T \nabla u) = \div\left(\frac{\Phi'(|\nabla u|)}{|\nabla u|}\nabla u\right) \enspace .
  \label{eq:divequality}
\end{equation} 
The equality in \eqref{eq:divequality} holds if 
\begin{equation}
  T \nabla u = \frac{\Phi'(|\nabla u|)}{|\nabla u|}\nabla u + C \enspace , % \nabla \times C \enspace %,
  \label{eq:divequality2}
\end{equation} 
for some vector-field $C$ where $\div(C) = 0$. In this work we consider $C = 0$. Here extending the equation with the scalar product from the left side, and integrating the result with respect to $|\nabla u|$ we find $\Phi$ as
\begin{equation}
  \Phi(|\nabla u|) = \int \frac{\nabla u^t T \nabla u}{|\nabla u|} \; d|\nabla u| \enspace .
  \label{eq:phiGeneral}
\end{equation}

Using this result, we motivate the subsequent theorem by performing the integration with a tensor defined by the outer product of the gradient, then
%\begin{equation}
$\Phi(|\nabla u|) = \int |\nabla u|^3\,d|\nabla u| = \frac{1}{4} |\nabla u|^4 +c$ %  \enspace ,
%\end{equation}
where $c$ is a constant. Setting the constant to null we obtain the tensordriven functional
\begin{equation} 
  J(u) = \iint_{\Omega} \frac{1}{2} (u-u_0)^2 \;dxdy + \lambda \iint_\Omega\frac{1}{4} \nabla u^t T \nabla u \,dxdy \enspace .
\end{equation}
Thus, a generalization of this energy functional is now stated in the following theorem.
\begin{theorem}
  Let $u_0$ be an observed image in a domain $\Omega \subset \mathbb{R}^2$, and let $T(u_x,u_y)$ describe the local orientation of the underlying structure of $u$ in $\Omega$. Denote by $J(u)$ the functional 
  \begin{equation}
    J(u) =  \frac{1}{p} ||u-u_0||^p + \lambda \iint_{\Omega} \nabla u^t T(\nabla u) \, \nabla u \,dxdy
    \label{eq:mainFunctional}
  \end{equation}
  where $u \in C^2$ and $1 < p < \infty$. Then the E-L equation of $J(u)$ is
%  \begin{empheq}[left={\empheqlbrace}]{align}
%  %   \left\{
% %      \begin{array}{r l}
%         (u - u_0)^{p-1} - \lambda \div \left( S \nabla u \right)  = 0  & \quad \text{in } \Omega \\[1mm]
%         \partial_n u  = 0                                    & \quad  \text{on } \partial \Omega 
% %      \end{array} % \right.
   
% \end{empheq}
 \begin{align}
    \left\{
      \begin{array}{r l}
%        (u - u_0)^{p-1} + \lambda \div \left( \begin{pmatrix} \nabla u^t T_{u_x} \nabla u \\  \nabla u^t T_{u_y} \nabla u  \end{pmatrix} + T\nabla u + T^t \nabla u \right) = 0  & \quad \text{in } \Omega \\[1mm]
        (u - u_0)^{p-1} - \lambda \div \left( S \nabla u \right)  = 0  & \quad \text{in } \Omega \\[1mm]
        \partial_n u  = 0                                    & \quad  \text{on } \partial \Omega  \\[1mm]
%        u(x,y,0) = u_0(x,y)                                  & \quad \text{in } \Omega    
      \end{array} \right.
    \label{eq:ELproposed}
  \end{align}
  where 
  \begin{equation}
  	S = \begin{pmatrix} \nabla u^t T_{u_x}  \\  \nabla u^t T_{u_y}  \end{pmatrix} + T + T^t \enspace ,
  \end{equation}
  and $T_{u_x}$, $T_{u_y}$ are the derivatives of $T(u_x, u_y)$ with respect to $u_x$ and $u_y$.
%  and $T = T(\nabla u) = T(u_x, u_y) = T(u_x,u_y)$.
\end{theorem}

\begin{proof}
To find the E-L equation of $J(u)$ we define the smoothness term
\begin{equation}
	\mathcal{R}(u) = \iint_{\Omega} \nabla u^t T(\nabla u) \, \nabla u \,dxdy \enspace .
\end{equation}
%Let $v \in C^\infty(\Omega)$ be an arbitrary vector such that $\partial_n v|_{\partial \Omega} = 0$. By the definition of the G{\^a}teaux derivative it follows that
Let $v \in C^\infty(\Omega)$ be an arbitrary function such that $\partial_n v|_{\partial \Omega} = 0$. By the definition of the G{\^a}teaux derivative it follows that
\begin{align}
\notag  \frac{\mathcal{R}(u + \varepsilon v) - \mathcal{R}(u)}{\varepsilon}  
= & \left \langle \nabla v, T(\nabla u + \varepsilon \nabla v)\nabla u \right \rangle + \left \langle \nabla u, T(\nabla u + \varepsilon \nabla v)\nabla v \right \rangle  \\
\notag & +  \frac{\left \langle \nabla u, (T(\nabla u + \varepsilon \nabla v) - T(\nabla u)) \nabla u \right \rangle}{\varepsilon} \\ 
& + \varepsilon \left \langle \nabla v, T(\nabla u + \varepsilon \nabla v)\nabla v \right \rangle \enspace ,
\end{align}
%= & \varepsilon \frac{\left (\nabla v, T(\nabla u + \varepsilon \nabla v)\nabla u \right ) + \left (\nabla u, T(\nabla u + \varepsilon \nabla v)\nabla v \right ) }{\varepsilon} \\
%& +  \frac{\left (\nabla u, (T(\nabla u + \varepsilon \nabla v) - T(\nabla u)) \nabla u \right )}{\varepsilon} \\ 
%& + \varepsilon^2 \frac{\left (\nabla v, T(\nabla u + \varepsilon \nabla v)\nabla v \right )}{\varepsilon} \\
where $\langle a,b \rangle = \iint_\Omega a^t b \; dxdy$ and $a,b \in \mathbb{R}^2 \rightarrow \mathbb{R}^n$ denote the scalar product on a function space. Note that, since $T(\nabla u) = T(u_x, u_y)$, % where $u_x=u_x$ and $u_y=u_y$, we get
\begin{align}
 \lim_{\varepsilon \rightarrow 0} &\frac{T(\nabla u + \varepsilon \nabla v) -T(\nabla u)}{\varepsilon} 
%     = \lim_{\varepsilon \rightarrow 0} \frac{T(u_x + \varepsilon v_x,u_y + \varepsilon v_y) -T(u_x,u_y)}{\varepsilon}  \\
%     = & \lim_{\varepsilon \rightarrow 0} \frac{T(u_x + \varepsilon v_x,u_y + \varepsilon v_y) -T(u_x,u_y)}{\varepsilon} 
         \implies T_{u_x} v_x + T_{u_y} v_y,  \text{ as } \varepsilon \rightarrow 0 \enspace .
\end{align}
% For simplicity, we denote $T_{u_x} v_x + T_{u_y} v_y$ by $T_u v = \frac{\partial T}{\partial u}v$. 
Then the variation of the smoothness term with respect to $v$ reads % in the direction of $v$ is
\begin{align}
  \langle \delta \mathcal{R}, v \rangle = \left \langle \nabla u, (T_{u_x} v_x + T_{u_y} v_y)\nabla u \right \rangle + \langle \nabla v, T \nabla u \rangle + \langle \nabla u, T \nabla v \rangle \enspace . % \;  \text{ as } \varepsilon \rightarrow 0 \enspace .
%  \frac{\partial \mathcal{R}}{\partial u} v = \left \langle \nabla u, (T_{u_x} v_x + T_{u_y} v_y)\nabla u \right \rangle + \langle \nabla v, T \nabla u \rangle + \langle \nabla u, T \nabla v \rangle,  \quad  \text{ as } \varepsilon \rightarrow 0 \enspace .
\end{align}
Using the scalar product and the adjoint operator we obtain
%By Green's formula obtain
\begin{align}
 \langle \delta \mathcal{R}, v \rangle = \left \langle \nabla v, (T_{u_x} u_x + T_{u_y} u_y)\nabla u \right \rangle + \langle \nabla v, T \nabla u \rangle + \langle \nabla v, T^t \nabla u \rangle \enspace . %\quad  \text{ as } \varepsilon \rightarrow 0 \enspace .
%  \frac{\partial \mathcal{R}}{\partial u} v = \left \langle \nabla v, (T_{u_x} u_x + T_{u_y} u_y)\nabla u \right \rangle + \langle \nabla v, T \nabla u \rangle + \langle \nabla v, T^t \nabla u \rangle, \quad  \text{ as } \varepsilon \rightarrow 0 \enspace .
\end{align}
By Green's formula we obtain
% Now computing Green's function, with boundary condition $(T\nabla u , n)|_{\partial \Omega} = 0$, in the first scalar product, whereas the two other required that $v|_{\partial \Omega} = 0$. Thus, the scalar products now read
\begin{align}
\label{eq:scalar1}  \langle \nabla v, T \nabla u \rangle  & = - \iint_\Omega v \; \div(T \nabla u) \; dxdy \enspace , \\
\label{eq:scalar2}  \langle \nabla v, T^t \nabla u \rangle  & = - \iint_\Omega v \; \div(T^t \nabla u) \; dxdy \enspace , \\
%\label{eq:scalar3}    \iint_\Omega \left \langle  \nabla v , \begin{pmatrix} \nabla u^t T_{u_x} \nabla u \\  \nabla u^t T_{u_y} \nabla u  \end{pmatrix} \right \rangle   \; dxdy & = - \iint_\Omega v \, \div \left ( \begin{pmatrix} \nabla u^t T_{u_x} \\  \nabla u^t T_{u_y}  \end{pmatrix}  \nabla u \right )\;  dxdy \enspace .
\label{eq:scalar3}    \left \langle  \nabla v , \begin{pmatrix} \nabla u^t T_{u_x} \nabla u \\  \nabla u^t T_{u_y} \nabla u  \end{pmatrix} \right \rangle   & = - \iint_\Omega v \, \div \left ( \begin{pmatrix} \nabla u^t T_{u_x} \\  \nabla u^t T_{u_y}  \end{pmatrix}  \nabla u \right )\;  dxdy \enspace .
\end{align}
Hence, using \eqref{eq:scalar1} - \eqref{eq:scalar3} now gives
\begin{equation}
  \langle \delta \mathcal{R}, v \rangle = - \iint_\Omega v \div(S \nabla u) \; dxdy \enspace .
\end{equation}
Thus   
\begin{equation}
  S = \begin{pmatrix} \nabla u^t T_{u_x}  \\  \nabla u^t T_{u_y}  \end{pmatrix} + T + T^t \enspace ,
  \label{eq:S}
\end{equation}
in the E-L equation \eqref{eq:ELproposed}. Now by taking the derivative of the data term in $J(u)$ with respect to $u$ the proof is completed. \qed
\end{proof}

\newcommand{\imgsize}{0.19}

\section{Necessary conditions for the existence of a variational formulation}
In this section we formulate necessary conditions for the existence of an energy functional of the type in \eqref{eq:mainFunctional} given a tensor-based diffusion PDE \eqref{eq:ELproposed}. 
\begin{theorem}
  \label{prop:1}
  For a given tensor $S = S(u_x,u_y)$ with entities $s^{(i)}$ satisfying the necessary condition (NC)
  \begin{equation}
    s^{(3)} + u_x s_{u_x}^{(3)} - u_x s_{u_y} ^{(1)} = s^{(2)} + u_y s_{u_y}^{(2)} - u_y s_{u_x}^{(4)} \enspace ,
    \label{eq:NC}
  \end{equation}
  there exists a symmetric tensor $T(u_x,u_y)$ such that the tensor-based diffusion PDE
  \begin{align}
    \left\{
      \begin{array}{r l}
        (u - u_0)^{p-1} - \div \left( S \nabla u \right ) = 0 & \quad \text{in } \Omega  \\[1mm]
        \partial_n u = 0                    & \quad  \text{on } \partial \Omega  \\[1mm]
        % u(x,y,0) = u_0(x,y)                 & \quad \text{in } \Omega  \\
      \end{array} \right. 
    \label{eq:EulerS}
  \end{align}
  where $1 < p < \infty $ is the E-L equation to the functional $J(u)$ in \eqref{eq:mainFunctional}. Moreover, the entities in 
  \begin{equation}
    T(u_x,u_y) = \begin{pmatrix} f(u_x,u_y) & g(u_x,u_y) \\ g(u_x,u_y) & h(u_x,u_y) \end{pmatrix} \enspace ,
  \end{equation}
  are given below. First we consider $g(u_x,u_y) = \tilde{g}(\alpha,\beta)$, where 
  \begin{equation}
    \label{eq:g} \tilde{g}(\alpha,\beta) = \frac{1}{\alpha^2}\int \alpha \tilde{R}(\alpha,\beta)\,d \alpha + \frac{1}{\alpha^2}\tilde{\theta}(\beta) \enspace ,
  \end{equation}
 and $\alpha = u_x$, $\beta = u_x/u_y$ and $\tilde{R}(\alpha,\beta) = R(u_x,u_y)$ is the right-hand side of NC \eqref{eq:NC}. Second, $f$ and $h$ are obtained from
 \begin{align}
   \label{eq:f} f(u_x,u_y) & = \frac{1}{u_x^2}\int_0^{u_x} (\xi s^{(1)}(\xi,u_y) - \xi u_yg_{\xi}(\xi,u_y))\,d\xi +\frac{1}{u_x^2}\rho(u_y) \enspace , \\
   \label{eq:h} h(u_x,u_y) & = \frac{1}{u_y^2}\int_0^{u_y} (\eta s ^{(4)} (u_x,\eta) - u_x \eta g_{\eta}(u_x,\eta))\,d\eta
    +\frac{1}{u_y^2}\zeta(u_x) \enspace ,
  \end{align}
  where $\rho$, $\zeta$ and $\tilde{\theta}$ are arbitrary functions. % and $R$ is the right-hand side of \eqref{eq:NC}. %, $R(\alpha,\beta) = r(u_x,u_y)$ and $\alpha=u_x$ and $\beta = u_x/u_y$, then 
\end{theorem}  
\begin{proof}
Let the system of equations, 
\begin{empheq}[left={\empheqlbrace}]{align}
  u_xf_{u_x} + u_yg_{u_x} + 2f &= s^{(1)} \label{eq:1} \\
  u_xg_{u_x} + u_yh_{u_x} + 2g &= s^{(2)} \label{eq:2} \\
  u_xf_{u_y} + u_yg_{u_y} + 2g &= s^{(3)} \label{eq:3} \\
  u_xg_{u_y} + u_yh_{u_y} + 2h &= s^{(4)} \label{eq:4}
\end{empheq}
represent the expansion of \eqref{eq:S} with given notation. The function $f$ in \eqref{eq:f} is obtained from \eqref{eq:1} by solving
\begin{equation}
  \frac{\partial}{\partial u_x}(u_x^2f)   = u_xs^{(1)} - u_xu_yg_{u_x} \enspace .
\end{equation}
To construct $g$, differentiate $f$ in \eqref{eq:f} with respect to $u_y$. This give that 
\begin{equation}
  f_{u_y}(u_x,u_y) = \frac{1}{u_x^2}\int_0^{u_x} (\xi s_y^{(1)}(\xi,u_y) 
  - \xi g_{\xi}(\xi,u_y) - \xi u_y g_{\xi u_y}(\xi,u_y) )\,d\xi
  +\frac{1}{u_x^2}\rho'(u_y) \enspace .
\end{equation}
Insert now $f_{u_y}$ into \eqref{eq:3} and differentiate with respect to $u_x$ to obtain 
\begin{equation}\label{e39}
  u_x g_{u_x} + u_yg_{u_y} + 2g  =  s^{(3)} + u_x s_{u_x}^{(3)} - u_x s_{u_y} ^{(1)} \enspace . 
\end{equation}
In the same way, solve for $h$ in \eqref{eq:2} to obtain \eqref{eq:h}. Using now \eqref{eq:2} in \eqref{eq:4} yield
\begin{equation}\label{e40}
  u_x g_{u_x} + u_yg_{u_y} + 2g  =  s^{(2)} + u_y s_{u_y}^{(2)} - u_y s_{u_x}^{(4)} \enspace .
\end{equation}
Comparing \eqref{e39} and \eqref{e40}, we deduce the necessary condition (NC)
%\begin{equation}
%    s^{(3)} + u_x s_{u_x}^{(3)} - u_x s_{u_y} ^{(1)} = s^{(2)} + u_y s_{u_y}^{(2)} - u_y s_{u_x}^{(4)}
%\end{equation}
to establish the existence of $g$, $f$ and $h$. From NC we are now able to compute $g$ satisfying
\begin{equation}
  u_x g_{u_x} + u_yg_{u_y} + 2g = R(u_x,u_y) \enspace ,
	\label{eq:R}
\end{equation} 
where $R$ is the right-hand side (or the left-hand side) in NC. Changing variables $\alpha =u_x$ and $\beta = u_x/u_y$, in \eqref{eq:R} we get %  With this change of variables the equation now reads
\begin{equation}
  \tilde{g}_\alpha(u_x\alpha_{u_x} + u_y\alpha_{u_y}) + \tilde{g}_\beta(u_x\beta_{u_x} + u_{y}\beta_{u_y}) + 2\tilde{g} = \tilde{R}(\alpha,\beta) \enspace ,
  \label{eq:Ruv}
\end{equation}
where $\tilde{R}(\alpha,\beta) = R(u_x,u_y)$ and $\tilde{g}(\alpha,\beta) = g(u_x,u_y)$. Since $\alpha_{u_x} = 1$, $\alpha_{u_y}=0$ and $\beta_{u_x}=1/u_y$ $\beta_{u_y}=-u_x/u_y^2$, equation \eqref{eq:Ruv} can then be writen as $\alpha \tilde{g}_\alpha + 2\tilde{g} = \tilde{R}$. 
%\begin{equation}
%  \alpha \tilde{g}_\alpha + 2\tilde{g} = \tilde{R}  \enspace .
%\end{equation}
We solve for $g$ in the following way
\begin{equation}
  \frac{\partial}{\partial \alpha}(\alpha^2\tilde{g}) = \alpha \tilde{R}(\alpha,\beta)
  \Leftrightarrow
  \tilde{g} = \frac{1}{\alpha^2}\int \alpha \tilde{R}(\alpha,\beta)\,d\alpha + \frac{1}{\alpha^2}\tilde{\theta}(\beta) \enspace ,
	\label{eq:gab}
\end{equation}
where $\tilde{\theta}$ is an arbitrary function. Put $\tilde{\theta}(\alpha,\beta) = \theta(u_x,u_y)$ to obtain $g$ in \eqref{eq:g}. Now $f$ and $h$ can be obtained from \eqref{eq:f} and \eqref{eq:h}. \qed
\end{proof}

\begin{corollary}
For the tensor $S = \nabla u \nabla u^t$ there exists a tensor
\begin{equation}
  T = 
  \frac{1}{4}
  \begin{pmatrix}
    u_x^2    &   u_xu_y \\
    u_xu_y  &   u_y^2
  \end{pmatrix}
  \enspace ,
  \label{eq:Tcor}
\end{equation}
such that the following PDE 
\begin{align}
  \left\{
    \begin{array}{r l}
      (u - u_0)^{p-1} - \div \left( S \nabla u \right ) = 0 & \quad \text{in } \Omega  \\[1mm]
      \partial_n u = 0                    & \quad  \text{on } \partial \Omega  \\[1mm]
      % u(x,y,0) = u_0(x,y)                 & \quad \text{in } \Omega  \\
    \end{array} \right. 
  \label{eq:EulerS_corollary}
\end{align}
is the E-L equation of the functional $J(u)$ in \eqref{eq:mainFunctional}. 
\end{corollary}
\begin{proof}
  It is straighforward to verify that the entities of the tensor $S$ satisfies the NC condition where the right-hand side \eqref{eq:NC}
  \begin{equation}
    R(u_x,u_y) = 2 u_x u_y \enspace .
    \label{eq:Rcor}
  \end{equation}
  For simplicity we choose $\theta$, $\rho$, and  $\zeta$ to be zero in \eqref{eq:g} - \eqref{eq:h}. With change of variables $\alpha = u_x$ and $\beta = u_x / u_y$,  \eqref{eq:Rcor} become $\tilde{R}(\alpha,\beta) = 2\alpha^2/\beta$ so that \eqref{eq:g} is  $\tilde{g}(\alpha,\beta) = \frac{\alpha^2}{2 \beta}$, then
  \begin{equation}
    g(u_x,u_y) = \frac{1}{2} u_xu_y \enspace .
    \label{eq:corr_g}
  \end{equation}
  Using \eqref{eq:f} and \eqref{eq:h}, $f$ and $h$ are
  \begin{align}
    \label{eq:corr_f} f(u_x,u_y) & = \frac{1}{u_x^2}\int_0^{u_x} (\xi^3 - \xi u_y^2)\,d\xi = \frac{u_x^2}{4} - \frac{u_y^2}{4} \enspace , \\
   \label{eq:corr_h} h(u_x,u_y) & = \frac{1}{u_y^2}\int_0^{u_y} (\eta^3 - \eta u_x^2 )\,d\eta = \frac{u_y^2}{4} - \frac{u_x^2}{4} \enspace .
  \end{align}
  Let $\widehat{T}$ be the tensor defined by \eqref{eq:corr_g} - \eqref{eq:corr_h}, then put this tensor in \eqref{eq:mainFunctional} and obtain
  \begin{equation}
    \nabla u^t \widehat{T} \nabla u =
    % \frac{1}{4}
    \nabla u^t
    \begin{pmatrix}
      u_x^2 - u_y^2    &   2u_xu_y \\
      2u_xu_y  &   u_y^2 - u_x^2
    \end{pmatrix}
    \nabla u
    = 
    % \frac{1}{4}
    \nabla u^t
    \begin{pmatrix}
      u_x^2    &   u_xu_y \\
      u_xu_y  &   u_y^2
    \end{pmatrix}
    \nabla u
    = 
    \nabla u^t T \nabla u
    \enspace ,
    \label{eq:Texpanded}
  \end{equation}
  neglecting the factor $1/4$. This concludes the proof and \eqref{eq:Tcor} follow from \eqref{eq:Texpanded}.
\end{proof}

\subsection{Weighted tensor-based variational formulation}
In this section we want to find a weighted tensor 
\begin{equation}
	T = \left ( w * \begin{pmatrix} f(u_x,u_y) & g(u_x,u_y) \\ g(u_x,u_y)  & h(u_x,u_y) \end{pmatrix} \right)(x,y) \enspace ,
	\label{eq:T2}
\end{equation}
where $*$ is the convolution operator and $w$ is a smooth kernel, such that the functional
\begin{equation}
    J(u) =  \frac{1}{p} ||u-u_0||^p + \lambda \iint_{\Omega} \nabla u^t T(\nabla u) \, \nabla u \,dxdy \enspace ,
    \label{eq:mainFunctional2}
\end{equation}
has 
  \begin{align}
    \left\{
      \begin{array}{r l}
%        (u - u_0)^{p-1} + \lambda \div \left( \begin{pmatrix} \nabla u^t T_{u_x} \nabla u \\  \nabla u^t T_{u_y} \nabla u  \end{pmatrix} + T\nabla u + T^t \nabla u \right) = 0  & \quad \text{in } \Omega \\[1mm]
        (u - u_0)^{p-1} - \lambda \div \left( S \nabla u \right)  = 0  & \quad \text{in } \Omega \\[1mm]
        \partial_n u  = 0                                    & \quad  \text{on } \partial \Omega  \\[1mm]
%        u(x,y,0) = u_0(x,y)                                  & \quad \text{in } \Omega    
      \end{array} \right.
    \label{eq:ELproposed2}
  \end{align}
as the corresponding E-L equation where $S = w * (\nabla u \nabla u^t)$.
\begin{proposition}
  For the tensor $S = w *(\nabla u \nabla u^t)$, the PDE \eqref{eq:ELproposed2} is the E-L
  equation of the functional \eqref{eq:mainFunctional2} with $T$ as in \eqref{eq:T2} if the necessary condition
  \begin{equation}
  u_x(x,y) (w * g_{u_x})(x,y) = u_y(x,y) (w * g_{u_y})(x,y) \enspace , %  \quad\Leftrightarrow\quad  
  % w * \left(u_x(x,y) g_{z_1}(z_1,u_y)) - u_y(x,y) g_{u_y}(z_1,u_y)\right)=0 \enspace .
  \label{eq:condTw}
\end{equation}
is satisfied. Moreover, the entities of $T$ can be obtained by solving the following system of differential equations
\begin{empheq}[left={\empheqlbrace}]{align}
  u_x (w * f_{u_x}) + u_y(w * g_{u_x}) +2 (w * f)  & = w * u_x^2     \label{eq:DES_1} \\
  u_x (w * g_{u_x}) + 2 (w * g)                    & = w * (u_xu_y)  \label{eq:DES_2} \\
  u_y (w * g_{u_y}) + 2 (w * g)                    & = w * (u_xu_y)  \label{eq:DES_3} \\
  u_x (w * g_{u_y}) + u_y (w * h_{u_y}) + 2 (w * h)& = w * u_y^2     \label{eq:DES_41}
\end{empheq}  
\end{proposition}
  % Following the proof of Theorem 2, the proof is straightforward with the difference of the smooth kernel $w$ which introduce some technical difficulties. 
%\begin{proof}
%  The difference from Theorem 2 is the smooth kernel $w$ which introduce some technical difficulties to complete the proof. Hence, we omit the proof and give \eqref{eq:condTw} as a necessary condition to derive a functional given a PDE on the form \eqref{eq:ELproposed2}.
%\end{proof}

\emph{Remark 1}: We have chosen to include $w$ in the formulation of the tensor \eqref{eq:T2}. This will allow us to include $u_x$ and $u_y$ inside the convolution operation in \eqref{eq:condTw} assuming that $u_x$ and $u_y$ are sampled at different positions on the grid.   
 
\emph{Remark 2}: Theorem 1, describe the case when $rank(S) = 1$, \emph{i.e.} the signal has an intrinsically one-dimensional structure (such as lines), Proposition 1 describe in general the case when $rank(S) = 2$. 

To conclude, we have derived a framework which states under what conditions a given PDE is the corresponding E-L equation to a tensor-based image diffusion energy functional. In the next section the derived E-L equation \eqref{eq:ELproposed} is used on a color image denoising problem.

% As previously, combining \eqref{eq:DES_2} and \eqref{eq:DES_3}, the necessary condition to solve the system of equations is Note that it is possible to include the gradients $u_x(x,y)$ and $u_y(x,y)$ in the convolution since it is defined on the surface in the region $\Omega$. If we were to consider the elements $f$, $g$ and $h$ as functions of points on the grid $\Omega$ the above condition would only be valid when the gradient of $u$ is uniform in all directions.

% computing a tensor in a functional on the form \eqref{eq:mainFunctional} such that the existence
% of a corresponding E-L (E-L) equation is guaranteed. 

% with a spatially averaged tensor $T = w * \nabla u \nabla u^t$.

% given an E-L equation we have stated necessary conditions to find a corresponding tensor valued energy functional of the form \eqref{eq:mainFunctional}. 
% Depending on the structure of the tensor in the E-L equation \eqref{eq:divS} 
% Depending on the type of tensor $S$ we give the necessary conditions for the derivations to hold (\emph{i.e.} \eqref{eq:condT} or \eqref{eq:condTw}) with these results we have shown that under given conditions a tensor valued functional exists on the form \eqref{eq:mainFunctional}.
% \begin{equation}
%   J(u) = \frac{1}{p}||u-u_0||^p + \int \nabla u^T T(\nabla u)\nabla u \; dxdy\enspace .
% \end{equation}

\section{Application to color image denoising}
It is known that the components of the RGB color space are correlated, therefore decorrelation methods using PCA, HSV, Lab \emph{etc.} have been investigated by the image processing community. In this work we utilize the decorrelation transform by Lenz and Carmona \cite{Lenz2010} which has shown to perform well for color image denoising problems \cite{diva2:424137}. The transform is described by the matrix
\begin{equation}
  \begin{pmatrix}
    I \\
    C_1 \\
    C_2
  \end{pmatrix}
  =
  \frac{1}{\sqrt{6}}
  \begin{pmatrix}
    \sqrt{2} &  \sqrt{2}  &  \sqrt{2} \\
    2        &  -1        &  -1       \\
    0        &  \sqrt{3}  &  -\sqrt{3}
  \end{pmatrix}
  \begin{pmatrix}
    R \\
    G \\
    B
  \end{pmatrix}
  % P = \frac{1}{\sqrt{3}}
  % \begin{pmatrix}
  %   1 			& 1							& 1 \\[0.3em]
  %   \sqrt{2}	& \sqrt{2}\cos(2\pi / 3)	& \sqrt{2}\cos(4\pi / 3) \\[0.3em]
  %   0			& \sqrt{2}\sin(2\pi / 3)	& \sqrt{2}\sin(4\pi / 3) \\
  % \end{pmatrix}
  \label{eq:p}
\end{equation}
and when applied to the RGB color space a primary component describing the average gray-value $I$ and two color-opponent components $C = \begin{pmatrix} C_1, & C_2 \end{pmatrix}$ are obtained.  

\subsection{From structure tensor to diffusion tensor}
\label{sec:struct_to_diff}
Let $T$ be defined as the outer product of the image gradient weighted by a smooth kernel, see \eqref{eq:Tw}. Since $T$ is symmetric it is also positive-semidefinite. To favor smoothing along edges the tensor is scaled using a diffusivity function $g$. The scaling affects the eigenvalues of the tensor, but preserves the eigenvectors since they describe the orientation of the image structure. In this study we use the negative exponential function $\exp(-ks)$ as the scaling function where $k$ is a positive scalar, \emph{i.e.}
\begin{equation}
  D(T) = g(T) = V g(\Lambda) V^{-1} \enspace ,
  \label{eq:eigdecomp}
\end{equation} 
where $V$ are the eigenvectors and $\Lambda$ the eigenvalues of $T$ \cite{5696757}. Then $\Lambda$ has the eigenvalues $\lambda_1$ and $\lambda_2$ on its main diagonal such that $g(\Lambda) = \mathrm{diag}(g(\lambda_1), g(\lambda_2))$. Then we denote $D = D(T)$ the diffusion tensor defined by \eqref{eq:eigdecomp}. With the same notation define $E = E(N)$, where $N = \begin{pmatrix} \nabla u^t T_{u_x}  \\  \nabla u^t T_{u_y} \end{pmatrix}$ in \eqref{eq:S}. Since $N$ is a nonsymmetric tensor, $E$ may be a complex valued function, however in practice $Im(E) \approx 0$ hence this factor is neglected. 

\subsection{Discretization}
To find the solution $u$ of \eqref{eq:ELproposed} with tensors $E(N)$ and $D(T)$ we define the parabolic equation
\begin{equation}
  \partial_t u - \div(E(N) \nabla u) - 2\div(D(T) \nabla u) = 0 \enspace .
\end{equation}
The term containing $D(T)$ is implemented by using the non-negativity scheme described in \cite{Weickert96anisotropicdiffusion} resulting in a term $A(u)$ whereas the other divergence term is implemented using the approach in \cite{Weickert_2002} described by $B(u)$. In an explicit discrete setting the iterative update equation become
\begin{equation}
  u^{i+1} = u^i + \tau( A(u^i) + B(u^i) ) \enspace ,
%  u^{i+1} = u^i + \tau( \div(E^i \nabla u^i) + 2\div(D^i \nabla u^i) )
\end{equation}
where $\tau$ is the steplength for each iteration $i$.

\subsection{Experiments}
The image noise variance, $\sigma_{est}^2$, is estimated by using the technique in \cite{Wolfgang2000} which is the foundation for the estimate of the diffusion parameter in \cite{5696757} (cf. (47) p. 6), the resulting diffusion parameter for the diffusion filters is
\begin{equation}
  k = \frac{e - 1}{e - 2}\sigma_{est}^2 \enspace ,
  \label{eq:k}
\end{equation}
where $e$ is the Euler number. A naive approach is taken to approximate the noise of the color image, the noise estimate will be the average sum of the estimated noise of each RGB color component. In most cases this approach works reasonably well for noise levels $\sigma < 70$ when compared to the true underlying noise distribution, but note that the image itself contains noise, hence the estimate will be biased. We use color images from the Berkeley segmentation dataset \cite{MartinFTM01}, more specifically images \emph{\{87065, 175043, 208001, 8143\}.jpg}, being labeled as lizard, snake, mushroom and owl (see figure \ref{fig:result}). All images predominately contain scenes from nature and hence large number of branches, leafs, grass and other commonly occurring structures. The estimated noise levels for the different images can be seen in table \ref{table:result}.

In this work we found that modifying the estimate of the diffusion parameter to scale with the color transform $k_I = 10^{-1}k$ and $k_c = 10 k$, yields an improved result compared to using $k$ without any scaling. The motivation for using different diffusion parameters in the different channels is that filtering in the average gray-value channel $I$ primarily affects the image noise and hence it is preferable to reduce the $k$-parameter in this channel. However, one must be careful since noise may be considered as structure and may be preserved throughout the filtering process. Filtering in the color opponent components will provide intra-region smoothing in the color space and hence color artifacts will be reduced with a less structural preserving diffusion parameter. 

To illustrate the capabilities of the proposed technique it is compared to three state-of-the-art denoising techniques, namely tracebased diffusion (TR) \cite{diva2:424137}, diffusion in the RGB color space \cite{Weickert1999201} and the color version of BM3D \cite{4378954}. Code was obtained from the authors of TR-diffusion and BM3D, whereas the RGB-color space diffusion was reimplemented based on \cite{Weickert1999201} and \cite{Weickert96anisotropicdiffusion} (cf. p. 95), the $k$-parameter in this case is set to the value obtained from \eqref{eq:k} since it uses the RGB color space. To promote research on diffusion techniques we intend to release the code for the proposed technique.

For a quantitative evaluation we use the SSIM \cite{Wang2004} and the peak-signal-to-noise (PSNR) error measurements. To the best of our knowledge there is no widely accepted method to evaluate color image perception, hence the error measurements are computed for each component of the RGB color space and averaged. Gaussian white noise with $\{5,10,20,50,70\}$ standard deviations (std) is added to the images. Before filtering is commenced the noise is clipped to fit the 8bit image value range. The steplength for the diffusion techniques was set to $0.2$ and manually scaled so that the images in figure \ref{fig:result} are obtained after approximately $1/2$ of the maximum number iterations (100).  

The purpose here is not to claim the superiority of any technique, rather it is to illustrate the advantages and disadvantages of the different filter methodologies in various situations. The error measurements SSIM and PSNR values are given to illustrate the strengths of the compared denoising techniques, results are shown in table \ref{table:result}.

\input{table_result}

It is clear from table \ref{table:result} that the established E-L equation performs well in terms of SSIM and PSNR values compared to other state-of-the-art denoising techniques used for color images. However, for images (such as the mushroom image) with large approximately homogeneous surfaces, BM3D is the favored denoising technique. This is due to many similar patches which the algorithm averages over to reduce the noise variance. On the other hand, diffusion techniques perform better in high-frequency regions due to the local description of the tensor. This is also depicted in figure \ref{fig:result} where the diffusion techniques preserve the image structure in the lizard, snake and the owl image, whereas BM3D achieves perceptually good results for the mushroom image.   

\section{Conclusion}
In this work we have given necessary conditions for a tensor-based image diffusion PDE to be the E-L equation for an energy functional. Furthermore, we apply the derived E-L equation in Theorem 1 to a color image denoising application with results comparable to state-of-the-art techniques. Generalization of Proposition 1 and its complete proof will be subject to further study.

\section*{Acknowledgment}
This research has received funding from the Swedish Research Council through a grant for the project \emph{Non-linear adaptive color image processing}, from the EC’s 7th Framework Programme (FP7/2007-2013), grant agreement 247947 (GARNICS), and by ELLIIT, the Strategic Area for ICT research, funded by the Swedish Government.

% This work has been conducted in collaboration with the Center for Medical Image Science and Visualization (CMIV) at Link{\"o}ping University, Sweden. CMIV is acknowledged for provision of financial support and access to leading edge research infrastructure.

\newcommand{\imgtype}[1]{neck_#1}
\newcommand{\imgstd}[1]{neck_#1}
\renewcommand{\imgsize}{0.19}

\newcommand{\imgpath}[3]{#1_#2_#3.pdf} % name,std,alg

\newcommand{\insertseq}[2]{
  \includegraphics[width=\imgsize\linewidth]{\imgpath{#1}{#2}{u1}}   & 
  \includegraphics[width=\imgsize\linewidth]{\imgpath{#1}{#2}{Tracebased}}    &
  \includegraphics[width=\imgsize\linewidth]{\imgpath{#1}{#2}{Extended}}      &
  \includegraphics[width=\imgsize\linewidth]{\imgpath{#1}{#2}{AD-RGB}}        &
  \includegraphics[width=\imgsize\linewidth]{\imgpath{#1}{#2}{CBM3D}}   
}

\begin{figure*}[p]
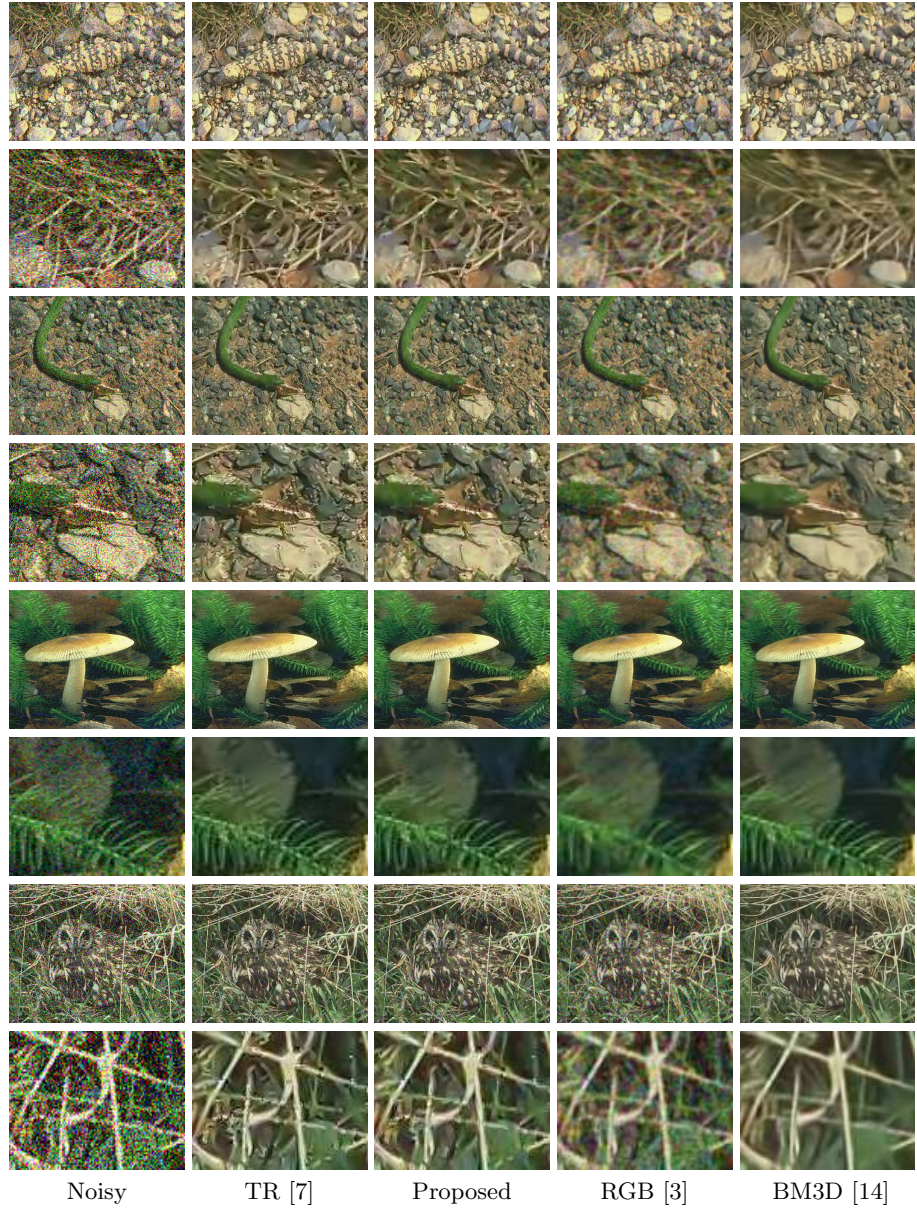

  \begin{center}
    \begin{tabular}{cccc cccc}
      %%%%%% 
      \insertseq{87065_lizard}{std50}              \\
      \insertseq{87065_lizard}{std50_zoom}         \\
      \insertseq{175043_snake_green}{std50}        \\
      \insertseq{175043_snake_green}{std50_zoom}   \\
      \insertseq{208001_mushroom}{std20}           \\
      \insertseq{208001_mushroom}{std20_zoom}      \\
      \insertseq{8143_owl}{std70}                  \\
      \insertseq{8143_owl}{std70_zoom}             \\
%      \insertseq{69015_panda}{std20}               \\
%      \insertseq{69015_panda}{std20_zoom}          \\
 %     Noisy  & {\AA}str{\"o}m et al. \cite{diva2:424137} & Proposed & Weickert \cite{Weickert1999201} & BM3D \cite{4378954} \\
      Noisy  & TR \cite{diva2:424137} & Proposed & RGB \cite{Weickert1999201} & BM3D \cite{4378954} \\
    \end{tabular}
  \end{center}
  \caption{Results. Even rows display some subregions of the image from the respectively previous row. The first four rows are illustrations using additive Gaussian noise of $50$ std, the next two rows of $20$ std, and the last two rows of 70 std noise. \emph{Best viewed in color.} }
  \label{fig:result}
\end{figure*}

\bibliographystyle{splncs}
\bibliography{egbib}

\end{document}

%% file: table_result.tex
\setlength{\tabcolsep}{4pt}
\begin{table}[t]
  \begin{center}
    \caption{SSIM and PSNR values }
    \label{table:result}
    \begin{tabular}{ll | cccc | cccc}
      \hline\noalign{\smallskip}
      & & \multicolumn{4}{c}{SSIM} & \multicolumn{4}{|c}{PSNR} \\
      Noise & Est & TR & Proposed & RGB & BM3D & TR & Proposed & RGB & BM3D \\
      \noalign{\smallskip}
      \hline
      \noalign{\smallskip}
      \multicolumn{2}{l}{Lizard} \\
      \noalign{\smallskip}
      \hline 
      \noalign{\smallskip}
      $5$  & $17.0$ & $\mathbf{0.99}$ & $\mathbf{0.99}$ & $0.98$ & $0.97$ & $\mathbf{38.0}$ & $\mathbf{38.0}$ & $35.2$ & $34.1$ \\
      $10$ & $19.2$ & $\mathbf{0.96}$ & $\mathbf{0.96}$ & $0.94$ & $0.95$ & $32.7$ & $\mathbf{32.8}$ & $30.6$ & $32.3$ \\
      $20$ & $24.2$ & $0.90$ & $0.91$ & $0.87$ & $\mathbf{0.92}$ & $27.3$ & $27.8$ & $26.7$ & $\mathbf{29.3}$ \\
      $50$ & $57.1$ & $0.75$ & $\mathbf{0.77}$ & $0.68$ & $0.73$ & $22.6$ & $\mathbf{23.5}$ & $22.0$ & $23.2$ \\
      $70$ & $67.1$ & $0.65$ & $\mathbf{0.70}$ & $0.59$ & $0.66$ & $20.8$ & $\mathbf{21.8}$ & $20.5$ & $21.7$  \\
      \noalign{\smallskip}
      \hline\noalign{\smallskip}
      \multicolumn{2}{l}{Snake} \\
      \noalign{\smallskip}
      \hline
      \noalign{\smallskip}
      $5$  & $15.5$ & $\mathbf{0.98}$ & $\mathbf{0.98}$ & $0.97$ & $0.96$ & $\mathbf{38.3}$ & $\mathbf{38.3}$ & $35.3$ & $34.6$ \\
      $10$ & $18.1$ & $\mathbf{0.95}$ & $\mathbf{0.95}$ & $0.92$ & $0.94$ & $33.0$ & $\mathbf{33.1}$ & $30.9$ & $32.6$ \\
      $20$ & $23.5$ & $0.87$ & $0.88$ & $0.83$ & $\mathbf{0.90}$ & $27.7$ & $28.4$ & $27.3$ & $\mathbf{29.7}$ \\
      $50$ & $53.9$ & $0.69$ & $\mathbf{0.73}$ & $0.63$ & $0.69$ & $23.2$ & $24.3$ & $23.0$ & $\mathbf{24.4}$ \\
      $70$ & $65.4$ & $0.59$ & $\mathbf{0.64}$ & $0.54$ & $0.61$ & $21.5$ & $22.7$ & $21.6$ & $\mathbf{22.9}$ \\     
      \noalign{\smallskip}
      \hline\noalign{\smallskip}
      \multicolumn{2}{l}{Mushroom} \\
      \noalign{\smallskip}
      \hline
      \noalign{\smallskip}
      $5$  & $9.9$ & $\mathbf{0.97}$ & $\mathbf{0.97}$ & $0.95$ & $0.96$ & $38.7$ & $\mathbf{38.9}$ & $36.9$ & $38.2$ \\
      $10$ & $13.9$ & $0.92$ & $\mathbf{0.93}$ & $0.89$ & $\mathbf{0.93}$ & $33.8$ & $34.3$ & $32.8$ & $\mathbf{35.1}$ \\
      $20$ & $21.1$ & $0.83$ & $0.86$ & $0.79$ & $\mathbf{0.88}$ & $29.3$ & $30.4$ & $29.2$ & $\mathbf{31.8}$ \\
      $50$ & $49.0$ & $0.66$ & $\mathbf{0.70}$ & $0.58$ & $\mathbf{0.70}$ & $25.1$ & $26.0$ & $24.7$ & $\mathbf{26.4}$ \\
      $70$ & $60.4$ & $0.57$ & $0.62$ & $0.52$ & $\mathbf{0.63}$ & $23.0$ & $23.8$ & $22.8$ & $\mathbf{24.1}$  \\
      \noalign{\smallskip}
      \hline\noalign{\smallskip}
      \multicolumn{2}{l}{Owl} \\
      \noalign{\smallskip}
      \hline
      \noalign{\smallskip}
      $5$  & $61.4$ & $\mathbf{0.99}$ & $\mathbf{0.99}$ & $0.98$ & $0.78$ & $36.9$ & $\mathbf{37.0}$ & $34.6$ & $23.2$ \\
      $10$ & $19.8$ & $\mathbf{0.97}$ & $\mathbf{0.97}$ & $0.95$ & $0.96$ & $\mathbf{33.0}$ & $\mathbf{33.0}$ & $30.1$ & $32.0$ \\
      $20$ & $25.7$ & $\mathbf{0.93}$ & $\mathbf{0.93}$ & $0.88$ & $\mathbf{0.93}$ & $27.8$ & $28.0$ & $25.8$ & $\mathbf{28.7}$ \\
      $50$ & $65.1$ & $0.80$ & $\mathbf{0.81}$ & $0.68$ & $0.72$ & $22.6$ & $\mathbf{23.0}$ & $20.5$ & $21.8$ \\
      $70$ & $71.2$ & $0.71$ & $\mathbf{0.73}$ & $0.59$ & $0.66$ & $20.4$ & $\mathbf{21.0}$ & $18.8$ & $20.5$ \\
      \noalign{\smallskip}
      \hline
      \noalign{\smallskip}
    \end{tabular}
  \end{center}
\end{table}
\setlength{\tabcolsep}{1.4pt}

%% file: eccv2012submission.bbl
\begin{thebibliography}{10}

\bibitem{Perona90scale-spaceand}
Perona, P., Malik, J.:
\newblock Scale-space and edge detection using anisotropic diffusion.
\newblock IEEE Trans. Pattern Analysis and Machine Intelligence \textbf{12}
  (1990)  629--639

\bibitem{Weickert96anisotropicdiffusion}
Weickert, J.:
\newblock Anisotropic Diffusion In Image Processing.
\newblock ECMI Series, Teubner-Verlag, Stuttgart, Germany (1998)

\bibitem{Weickert1999201}
Weickert, J.:
\newblock Coherence-enhancing diffusion of colour images.
\newblock Image and Vision Computing \textbf{17} (1999)  201--212

\bibitem{661181}
Sochen, N., Kimmel, R., Malladi, R.:
\newblock A general framework for low level vision.
\newblock Image Processing, IEEE Trans. \textbf{7} (1998)  310 --318

\bibitem{Renner2003}
Renner, A.I.:
\newblock Anisotropic Diffusion in Riemannian Colour Space.
\newblock PhD thesis, Ruprecht-Karls-Universit{\"a}t, Heidelberg (2003)

\bibitem{918563}
Tang, B., Sapiro, G., Caselles, V.:
\newblock Color image enhancement via chromaticity diffusion.
\newblock IEEE Trans. Image Processing \textbf{10} (2001)  701 --707

\bibitem{diva2:424137}
{\AA}str{\"o}m, F., Felsberg, M., Lenz, R.:
\newblock {Color Persistent Anisotropic Diffusion of Images}.
\newblock In: Image Analysis. Volume 6688 of LNCS., Springer (2011)  262--272

\bibitem{Lenz2010}
Lenz, R., Carmona, P.L.:
\newblock Hierarchical s(3)-coding of rgb histograms.
\newblock In: Selected papers from VISAPP 2009. Volume~68.
\newblock Springer (2010)  188--200

\bibitem{springerlink:10.1007/BF00131148}
Black, M.J., Rangarajan, A.:
\newblock On the unification of line processes, outlier rejection, and robust
  statistics with applications in early vision.
\newblock IJCV \textbf{19} (1996)  57--91

\bibitem{1238435}
Scharr, H., Black, M., Haussecker, H.:
\newblock Image statistics and anisotropic diffusion.
\newblock In: Ninth IEEE ICCV. (2003)  840 --847 vol.2

\bibitem{5539959}
Krajsek, K., Scharr, H.:
\newblock Diffusion filtering without parameter tuning: Models and inference
  tools.
\newblock In: CVPR. (2010)  2536 --2543

\bibitem{Baravdish11}
Baravdish, G., Svensson, O.:
\newblock Image reconstruction with p(x)-parabolic equation.
\newblock In: 7th ICIPE, Orlando Florida. (2011)

\bibitem{5696757}
Felsberg, M.:
\newblock Autocorrelation-driven diffusion filtering.
\newblock IEEE Trans. Image Processing \textbf{20} (2011)  1797--1806

\bibitem{4378954}
Dabov, K., Foi, A., Katkovnik, V., Egiazarian, K.:
\newblock Color image denoising via sparse 3d collaborative filtering with
  grouping constraint in luminance-chrominance space.
\newblock In: Image Processing, IEEE International Conference. (2007)  313--316

\bibitem{Forstner87}
F{\"o}rstner, W., G{\"u}lch, E.:
\newblock A fast operator for detection and precise location of distinct
  points, corners and centres of circular features.
\newblock In: ISPRS Intercommission, Workshop, Interlaken, pp. 149-155. (1987)

\bibitem{Weickert_2002}
Weickert, J.:
\newblock A scheme for coherence-enhancing diffusion filtering with optimized
  rotation invariance.
\newblock VCIR \textbf{13} (2002)  103--118

\bibitem{Wolfgang2000}
F{\"o}rstner, W.:
\newblock Image preprocessing for feature extraction in digital intensity,
  color and range images.
\newblock In: Geomatic Method for the Analysis of Data in the Earth Sciences.
  Volume~95 of LNES.
\newblock (2000)  165--189

\bibitem{MartinFTM01}
Martin, D., Fowlkes, C., Tal, D., Malik, J.:
\newblock A database of human segmented natural images and its application to
  evaluating segmentation algorithms and measuring ecological statistics.
\newblock In: Proc. ICCV. Volume~2. (2001)  416--423

\bibitem{Wang2004}
Wang, Z., Bovik, A., Sheikh, H., Simoncelli, E.:
\newblock Image quality assessment: from error visibility to structural
  similarity.
\newblock IEEE Trans. Image Processing \textbf{13} (2004)  600--612

\end{thebibliography}
